\documentclass[twocolumn, switch]{article} 
\usepackage{preprint}

\usepackage{microtype}
\usepackage{graphicx}
\usepackage{subcaption}
\usepackage{booktabs} 

\usepackage{amsmath, amsthm, amssymb, amsfonts}

\usepackage[numbers,square]{natbib}
\usepackage[utf8]{inputenc}	
\usepackage[T1]{fontenc}	
\usepackage{xcolor}		
\usepackage[colorlinks = true,
            linkcolor = purple,
            urlcolor  = blue,
            citecolor = cyan,
            anchorcolor = black]{hyperref}	
\usepackage{booktabs} 		
\usepackage{nicefrac}		
\usepackage{microtype}		
\usepackage{lineno}		
\usepackage{float}			

\usepackage{lipsum}		

\usepackage{newfloat}
\DeclareFloatingEnvironment[name={Supplementary Figure}]{suppfigure}
\usepackage{sidecap}
\sidecaptionvpos{figure}{c}

\usepackage{titlesec}
\titlespacing\section{0pt}{12pt plus 3pt minus 3pt}{1pt plus 1pt minus 1pt}
\titlespacing\subsection{0pt}{10pt plus 3pt minus 3pt}{1pt plus 1pt minus 1pt}
\titlespacing\subsubsection{0pt}{8pt plus 3pt minus 3pt}{1pt plus 1pt minus 1pt}

\usepackage{tikz,xcolor,hyperref}

\definecolor{lime}{HTML}{A6CE39}

\def\TA{\mathbf A}
\def\bb{\mathbf b}

\def\bd{\mathbf d}

\def\TE{\mathbf E}
\def\be{\mathbf e}

\def\TI{\mathbf I}

\def\by{\mathbf y}

\def\TW{\mathbf W}
\def\bv{\mathbf v}

\def\bx{\mathbf x}

\def\by{\mathbf y}

\def\Id{\operatorname{Id}}
\newcommand{\grad}{\ensuremath{\nabla}}

\theoremstyle{plain}
\newtheorem{theorem}{Theorem}[section]
\newtheorem{proposition}[theorem]{Proposition}

\theoremstyle{definition}

\theoremstyle{remark}

\title{CQnet: convex-geometric interpretation and constraining neural-network trajectories}

\usepackage{xwatermark}
\newwatermark[firstpage,color=gray!90,angle=0,scale=0.28, xpos=0in,ypos=-5in]{*correspondence: \texttt{bas@compgeoinc.com}}

\usepackage{authblk}

\author[1\thanks{\tt{bas@compgeoinc.com}}]{Bas Peters}

\affil[1]{Computational Geosciences Inc.}

\begin{document}

\twocolumn[ 
  \begin{@twocolumnfalse} 

\maketitle

\begin{abstract}
We introduce CQnet, a neural network with origins in the CQ algorithm for solving convex split-feasibility problems and forward-backward splitting. CQnet's trajectories are interpretable as particles that are tracking a changing constraint set via its point-to-set distance function while being elements of another constraint set at every layer. More than just a convex-geometric interpretation, CQnet accommodates learned and deterministic constraints that may be sample or data-specific and are satisfied by every layer and the output. Furthermore, the states in CQnet progress toward another constraint set at every layer. We provide proof of stability/nonexpansiveness with minimal assumptions. The combination of constraint handling and stability put forward CQnet as a candidate for various tasks where prior knowledge exists on the network states or output.
\end{abstract}
\vspace{0.35cm}

  \end{@twocolumnfalse} 
] 

\section{Introduction}
The ubiquitous success of neural networks in fields like computer vision, science and engineering, and control automatically raises the desire for more intuition, control, and interpretability of its operation, as well as provable properties of the networks. 

Regarding the intuition of how a trained neural network operates, classic and newer analyses for classification problems show how network layers add new decision boundaries, as viewed in terms of the input data space \citep{NIPS1987_4e732ced,266507,van2016some,NEURIPS2019_0801b20e}. A different type of intuition of how data propagates (flows) through a network (its trajectories), is provided by neural ordinary (and partial) differential equations (Neural ODEs) where the underlying continuous-time ODE, together with a discretization, prescribe (in)stability, energy conservation, invertibility similar to physical phenomena \citep{HaberRuthotto2017a,weinan2017proposal,Chang2017Reversible,pmlr-v80-lu18d,neuroODE,RuthottoHaber2018}. 

Here, we propose a new network design, \emph{CQnet}, based on the CQ algorithm \cite{Byrne_2002} for the convex split-feasibility problem (SFP). The proposed network offers new insights into the operation of a network, as well as some provable properties. Specifically we
\begin{itemize}
\item propose CQnet and illustrate its interpretation as trajectories that are seeking feasibility with respect to a layer-dependent convex set while being elements of some other constraint sets.

\item illustrate that CQnet's most basic form naturally includes two types of (possibly data-dependent or sample-specific) constraints on the trajectories: 1) constraints that are satisfied at every layer and the output; 2) constraint sets towards which the trajectories make progress during forward propagation. CQnet can also operate with a mix of deterministic and learned constraint sets.

\item require only minimal assumptions to prove the stability of CQnet in a nonexpansive sense by re-purposing ingredients from the convergence proofs of the CQ algorithm.

\end{itemize}

Besides interpreting CQnet's inner workings, there are a few other ways to relate CQnet to existing works. Related work includes the observation that neural networks can be built from nonexpansive and averaged operators \citep{combettes2020deep,doi:10.1137/19M1272780,hasannasab2020parseval}, who focus on layered networks with activation + affine structure. \cite{9054731} add constraints on the network weights during the training to ensure nonexpansiveness of standard non-linear+affine composite layers, and \cite{NEURIPS2020_798d1c28,heaton2021feasibility} construct implicit fixed point models. In this work, we prove the nonexpansiveness/stability of CQnet, or robustness to perturbations. While we also provide a layer-wise condition on the network weights, which is cheap to compute and easy to enforce, estimating entire network Lipschitz constants (see, e.g., \citep{SzegedyZSBEGF13,NEURIPS2019_95e1533e,NEURIPS2018_48584348,doi:10.1137/19M1272780,shi2022efficiently}) is not the scope of this work. 

Another line of work concerns neural networks with various types of state constraints. Satisfying constraints on the output of a network is possible either via the projection of the output of the last layer \citep{dalal2018safe}, or by training the network such that the output satisfies certain properties, for instance via alternating optimization schemes \cite{NIPS2019_9385,Pathak_2015_ICCV,10.1109/ICCV.2015.203,herrmann2019learned}, penalties added to the loss \citep{KERVADEC201988,kervadec2020bounding}, Lagrange multipliers \citep{NIPS1987_a87ff679,marquez2017imposing}, or by training the network via a feasibility problem \cite{stewart2017label,peters2022point}. Optimizing a network to satisfy constraints does not provide strict guarantees that validation samples will satisfy those constraints, particularly for very small datasets. For this reason, \cite{Brosowsky_Keck_Dunkel_Zollner_2021} propose a method for sample-specific constraints, and \cite{boesen2022neural} present an approach for constraining trajectories and outputs based on differential-algebraic equations.

Furthermore, while the proposed CQnet is part of the family of deep unrolled existing algorithms, e.g., \citep{lsp,8683124,9363511,Bertocchi_2020}, the overall goal is \emph{not} unrolling yet-another-algorithm to hopefully obtain (minor) computational advantages or image reconstruction improvements. Instead, we focus on the interpretation, the multiple ways to employ CQnet for constraining trajectories, and a few provable properties.

The paper proceeds by reviewing the CQ algorithm and a version of its convergence proof. This sets the stage to introduce the main contribution: CQnet and a corresponding proof of stability, as well as some variations and connections to other neural network types. Examples in small and large data classification and optimal control illustrate a couple of ways how CQnet can be employed while highlighting a combination of capabilities that most other networks cannot offer.

\section{Preliminaries and the CQ algorithm for convex split-feasibility}
A subset of convex optimization problems may be conveniently formulated as the convex split-feasibility problem (SFP) \cite{1084541,censor1994multiprojection}
\begin{equation}\label{cvx_feas}
\text{find} \:\: \bx \: \text{s.t.} \: \TA \bx \in Q \: \text{and} \: \bx \in C,
\end{equation}
where $\bx \in \mathbb{R}^n$ are the optimization variables, $\TA \in \mathbb{R}^{m \times n}$ is a matrix, and $Q$ and $C$ are closed and convex sets. 

While there are many ways to solve the SFP, the CQ algorithm \cite{Byrne_2002} offers an approach that avoids matrix inverses of $\TA$ and breaks down potentially difficult-to-compute projections into `simple' projections and matrix-vector products. The method minimizes over the set $C$ the objective function
\begin{equation}\label{CQ_obj}
f(\bx) =  \min_{\bx \in C} \frac{1}{2} \| P_Q(\TA \bx) - \TA \bx \|_2^2,
\end{equation}
where $P_Q $ is the Euclidean projection from $\mathbb{R}^n$ onto the closed and convex set $Q$, i.e.,
\begin{equation}\label{proj}
P_Q(\bx) = \arg\min_\by \frac{1}{2} \| \by - \bx \|_2^2 \:\: \text{s.t.} \:\: \by \in Q.
\end{equation}
The objective $f(\bx)$ is the squared Euclidian distance from $\bx$ to the constraint set $\{ \bx \: | \: \TA \bx \in Q \}$ restricted to $C$. The continuously differentiable (squared) distance function $d^2(\mathbf{y}): \mathbb{R}^n \rightarrow \mathbb{R}$
\begin{equation}\label{distsq}
d^2(\mathbf{y}) =  \frac{1}{2} \|  P_Q(\TA \by) - \TA \by \|_2^2,
\end{equation}
has a closed-form gradient \citep[Ex. 3.3]{hiriart2012fundamentals}
\begin{equation}\label{distsq_grad}
\grad_\by d^2(\mathbf{y}) =  \TA^\top (\Id - P_Q)\TA \by,
\end{equation}
with $\Id$ as the identity operator. Lipschitz continuity with parameter $\lambda$ for any $\bx \in \mathbb{R}^n$ and $\by \in \mathbb{R}^n$ implies
\begin{equation}\label{Lip}
\| f(\bx) - f(\by) \|_2 \leq \lambda \| \bx - \by \|_2,
\end{equation}
and is called nonexpansive if $\lambda \leq 1$, contractive if $\lambda < 1$. For the squared distance function \eqref{distsq} including linear operators $\TA$, we have $\lambda = \rho(\TA^\top \TA)$ as the largest eigenvalue \citep[Lemma 8.1]{CharlesByrne2004}.

An operator $T$ is $\alpha$-averaged if it can be written as the sum of the identity operator $\Id$ and a nonexpansive operator $V$,
\begin{equation}\label{av_op}
T = \alpha \Id + (1-\alpha) V,
\end{equation}
for $\alpha \in (0,1)$. The CQ algorithm finds a solution to \eqref{cvx_feas} by performing projected-gradient descent with stepsize $\alpha \in (0, 2 / \lambda)$ on the distance function, resulting in the iteration
\begin{equation}\label{CQ_iter}
\bx_{k+1}  = P_ C \big( \bx_k - \alpha \TA^\top (\Id - P_Q) \TA \bx_k  \big),
\end{equation}
This shows that $\bx_k$ is always feasible with respect to $C$, and descends towards the constraint set $\{ \bx \: | \: \TA \bx \in Q\}$.

The CQ algorithm found uses in many applications, including radiation therapy treatment planning \cite{Censor_2006, https://doi.org/10.1111/itor.12929}, compressed sensing \citep{Lopez_2012,gibali2018note}, Gene regulatory network inference \cite{Wang_2017}, image deblurring \citep{shehu2021new}, and adding constraints to the output of neural networks \cite{peters2022point}. 

The convergence of the CQ algorithm (e.g., \cite{CharlesByrne2004}) and its connection to forward-backward splitting \citep{comb_PFwdBwdS} is the cornerstone of several properties of the neural network proposed in this work. Here, we state a version of the proof so that elements will be reused in our theorem in a later section.

\begin{theorem}[Convergence of the CQ iteration]\label{thm:cq_converg}
If $\mathcal{X} = C \bigcap \{ \bx \: | \: \TA \bx \in Q \} \neq \emptyset$, let $\lambda$ be the largest eigenvalue of $\TA^\top \TA$, and $\alpha \in (0, 2 / \lambda)$, then from any initial guess the iteration \eqref{CQ_iter} converges to a solution in the sense that at iteration $k$ we have $\operatorname{inf}_{\by \in \mathcal{X}} \| \by -  \bx_k \|_2 \rightarrow 0$ monotonically non-increasing.
\end{theorem}
\begin{proof}
Rewrite the CQ iteration \eqref{CQ_iter} as the fixed point iteration $\bx_{k+1}  =\mathcal{P}_C \big(  \bx_k - \alpha  \TA^\top (I -P_Q) \TA \bx_k ) \leftrightarrow \bx_{k+1}  = T(\bx_k)$. According to Krasnosel'ski\u{\i}\textendash{}Mann \citep{krasnosel1955two,mann1953mean} (see also, e.g., \citep[Thm. 1]{ryuyin2022}), this iteration converges to a fixed point of $T$ if it is an averaged operator.

To show this, write as a composition $T=P_C (I - \alpha B)$, where the projection operator $P_C$ is nonexpansive and averaged \citep[Sect. 3.1]{ryu2016primer}, and $B=\grad_\by d^2(\mathbf{y})$ as in \eqref{distsq_grad} is $\lambda$-Lipschitz continuous. Therefore, $ (I - \alpha B)$ \citep[Sect. 5]{ryu2016primer} is averaged for $\alpha \in (0, 2 / \lambda)$. The composition $T$ is thus averaged, see, e.g., \citep[Thm. 27]{ryuyin2022}.
\end{proof}

\section{CQnet}
The goal is to construct a neural network that offers a clear interpretation of its inner workings in terms of convex geometry. Its foundation in constrained optimization equips the proposed network with constraints on the state trajectories and output.

We modify the CQ algorithm into a neural network by making the matrix $\TA$ iteration dependent ($\TA_k$) and learnable; network states $\bx_k$ are not learnable; initial guess $\bx_1$ corresponds to the input data $\bd$. Each CQ iteration is equivalent to a layer in the neural network. The last important item is that the projection operators $P_{C_k}$ and $P_{Q_k}$ are now layer dependent (not learnable in this work) and function similarly to an activation function. For instance, the widely used ReLU activation $\text{ReLU}(\bx) = \max(\bx, 0)$ is equivalent to the projection onto the set of elementwise bound constraints $\{ \bx \: | \: \bx \geq 0 \}$. The CQnet with $f$ layers thus takes the form
\begin{align}\label{CQnet}
&\bx_{k+1}  =P_{C_k} \big( \bx_k - \alpha  \TA_k^\top (\Id - P_{Q_k}) \TA_k \bx_k \big) \\
&\text{for} \: k=1,2,\cdots,f \quad \text{and} \quad \bx_1 = \bd \nonumber
\end{align}
At every layer, CQnet's states $\bx_k$ take a step toward the constraint set $\{ \bx \: | \: \TA_k \bx \in Q_k \}$, which changes at every iteration. In other words, the network states follow the learnable constraint set and stop moving whenever they become elements of the set. Furthermore, the states are always feasible with respect to the set $C_k$. Note that if we would allow for a slight generalization of the CQ algorithm and look at proximal maps instead of just projections, a larger number of common activation functions \citep[sect. 2.1]{combettes2020deep} fit in the proposed framework.

\subsection{Stability}
To prove the stability of CQnet, we borrow from the convergence of the CQ algorithm (Theorem \ref{thm:cq_converg}).

\begin{theorem}[Stability of the forward propagation of CQnet \eqref{CQnet}] \label{CQnet_stab}
Let $g(\{\alpha_k\},\{\TA_k\},\bd)$ describe the forward propagation of data $\bd$ through CQnet \eqref{CQnet} with learned weights $\{\TA_k\}$ for $i=1,2,\cdots,f$ layers with stepsizes $\alpha_k \in (0, 2 / \lambda_k ] \:\: \forall k$ and $\lambda_k = \rho(\TA_k^\top \TA_k)$ is the largest eigenvalue of $\TA_k^\top \TA_k$. Then CQnet is stable in the nonexpansive sense $\| g(\{\alpha_k\},\{\TA_k\},\bd_1) - g(\{\alpha_k\},\{\TA_k\},\bd_2) \| \leq \| \bd_1 - \bd_2 \|$.
\end{theorem}
\begin{proof}
Write a single CQnet layer $\bx_{k+1}  = P_{C_k} \big(  \bx_k - \alpha_k  \TA_k^\top (\Id -P_{Q_k})(\TA_k \bx_k) ) \leftrightarrow \bx_{k+1}  = T(\bx_k)$ as a fixed-point iteration. Forward propagation then amounts to a composition of $f$ operators $T_f \cdots T_2 T_1 \bd_1 = g(\{\alpha_k\},\{\TA_k\},\bd_1)$.

Compose each operator as $T_k = P_{C_k} (\Id - \alpha_k B_k)$, where the projection operator $P_{C_k}$ is nonexpansive (and averaged) (e.g., \citep[Sect. 3.1]{ryu2016primer}) and $B_k=\grad_\by d^2(\mathbf{y},\TA_k)$ as in \eqref{distsq_grad} is $\lambda_k$-Lipschitz continuous with $\lambda_k = \rho(\TA_k^\top \TA_k)$ as the largest eigenvalue \citep[Lemma 8.1]{CharlesByrne2004}. Therefore, $ (\Id - \alpha_k B_k)$ has the Lipschitz constant $\operatorname{max}(1,|1-\alpha_k \lambda_k |)$. Nonexpansiveness of $(\Id - \alpha_k B_k)$ follows directly for $\alpha_k \in (0, 2 / \lambda_k ]$ .

The product of the two nonexpansive operators $T_k = P_{C_k} (\Id- \alpha_k B_k)$ is also nonexpansive. By induction, the product of $f$ nonexpansive operators is again nonexpansive and we obtain $\| T_f \cdots T_2 T_1 \bd_1  -T_f \cdots T_2 T_1  \bd_2 \|_2 \leq  \| \bd_1 - \bd_2 \|$.
\end{proof}

The above proof is quite flexible in the sense that there are no assumptions like constant or `slowly' varying network weights, as sometimes found in continuous-time ODE-based stability proofs. The proof also includes state/layer normalization, as will be shown in section \ref{sect:someconstraints} below. Theorem \ref{CQnet_stab} also allows for different projections/activations per layer. In the next section, we take a closer look at the relation between $\alpha_k$ and $\TA_k$.

\subsection{Practical certificate of nonexpansiveness}
Theorem \ref{CQnet_stab} ensures stability (nonexpansiveness) of the network for sufficiently small $\alpha_k$, given $\TA_k$. In practice, however, we may desire guarantees that the network is nonexpansive during and after training, for a fixed $\alpha_k$.

Our starting point is that we need an upper bound on the largest eigenvalue $\lambda_k = \rho(\TA_k^\top \TA_k)$. For small matrices $\TA_k$ we can directly compute the largest eigenvalue. The more challenging case is when convolutional kernels parameterize $\TA$ with multiple input and output channels. \cite{sedghi2018the} present a method to compute the eigenvalues in this case, based on FFTs and SVDs, and its computational cost depends on $n^2$, making the computation expensive for large-scale inputs. To proceed with deriving an upper bound on the eigenvalues that is in closed form, non-asymptotic, and independent of $n$, we leverage work on stepsize selection for the CQ algorithm, see, e.g., \citep{Byrne_2002,cegielski2007convergence,Qu_2005,YANG2005166,wang2011choices,Lopez_2012,shehu2021new}. Although those works aim for stepsizes that accelerate convergence or avoid matrix inverses and eigenvalue computations, we repurpose the results to obtain a certificate of nonexpansiveness of CQnet (Proposition \ref{upper_bound_ev}) with very little (computational) effort. 

\cite{Byrne_2002} presented a bound for dense and sparse matrices. Below we specialize that bound to matrices where each block is a different Toeplitz matrix generated by a convolutional kernel (the blocks on a diagonal are not the same).

\begin{proposition}[Upper bound on the largest eigenvalue of $\TA(\Theta)^\top \TA(\Theta)$]\label{upper_bound_ev}
Let $\TA(\Theta)$ be a matrix with $c_\text{out} \times c_\text{in}$ Toeplitz structured blocks, where $c_\text{out}$ and $c_\text{in}$ are the number of output and input channels respectively. Each block is parameterized by a different convolutional kernel $\theta^{i,j} \in \mathbb{R}^{w\times w}$ with $w^2$ elements. Also assume that the convolutional kernels are normalized by the Euclidean length of all filters in the same block-row of $\TA(\Theta)$, i.e., $\tilde{\theta}^{i,j} = \theta^{i,j} / \sqrt{\sum_{j=1}^{c_\text{in}} \sum_{k=1}^w \sum_{l=1}^w (\theta^{i,j}_{k,l})^2)}$. Then $\rho(\TA(\tilde{\Theta})^\top \TA(\tilde{\Theta})) \leq w^2 c_\text{out}$
\end{proposition}

\begin{proof}
For $\TA \in \mathbb{R}^{M \times N}$, each entry of the vector of row-squared-sums $\bv \in \mathbb{R}^M$ is given by $\bv_m = \sum_{n=1}^N \TA_{m,n}^2$. Let $\TE \in \mathbb{R}^{M \times N}$ with $\TE_{m,n} = 1$ if $\TA_{m,n} \neq 0$ and $0$ otherwise. Then the eigenvalues of $\TA^\top \TA$ are bouned as $\rho(\TA^\top \TA) \leq  \operatorname{max}(\bv^\top \TE)$ \citep[Prop. 4.1]{Byrne_2002}. For a block matrix that is parameterized by convolutional kernels, $\TA(\Theta)$, each row of each block contains all elements of the corresponding convolutional kernel. Normalizing each convolutional kernel by the Euclidean length of all kernels in the same block-row via $\tilde{\theta}^{i,j} = \theta^{i,j} / \sqrt{\sum_{j=1}^{c_\text{in}} \sum_{k=1}^w \sum_{l=1}^w (\theta^{i,j}_{k,l})^2)}$, leads to $\bv_m = 1$ for all rows. Thus, $\bv^\top \TE$ counts the number of nonzero entries in each column of $\TA(\tilde{\Theta})$, where $\tilde{\Theta}$ indicates the normalized kernels. A matrix $\TA(\Theta)$ where each block is Toeplitz, has the same number of nonzeros in every column if all convolutional kernels are of the same size, as we assumed. Every column in $\TA(\Theta)$ then contains $w^2$ nonzeros in each of the $c_\text{out}$ blocks, so we obtained $\rho(\TA(\tilde{\Theta})^\top \TA(\tilde{\Theta})) \leq w^2 c_\text{out}$
\end{proof}

The above proposition provides a recipe to train CQnet with a certificate of nonexpansiveness, by normalizing the convolutional kernels according to Proposition \ref{upper_bound_ev}, at each training iteration. This normalization together with a combination of the stability condition $\alpha_k \in (0, 2 / \lambda_k ]$ and the bound $\rho(\TA(\tilde{\Theta})^\top \TA(\tilde{\Theta})) \leq w^2 c_\text{out}$ tells us to select $\alpha_k \leq 2/ (w^2 c_\text{out})$. The normalization corresponds to an $\ell_2$ norm constraint on every set of convolutional kernels as organized in the block row of the matrix representation of all convolutional kernels. The normalization is then the projection as used in a stochastic projected gradient algorithm. 

The bound and projection from Proposition \ref{upper_bound_ev} are in closed form, non-asymptotic, and computationally cheap (often $w \leq 7$ and $c_\text{out}$ is a few dozen of few hundred per layer). While there is no claim that this bound is particularly tight, the computational simplicity is appealing, compared to some other conditions, for different network designs, that require iterative algorithms like Dykstra \cite{sedghi2018the}, Douglas-Rashford, \citep{9054731}, Bjorck Orthonormalization{ \citep{pmlr-v97-anil19a}, or approaches that require training \citep{pmlr-v70-cisse17a,NEURIPS2018_48584348} to achieve Lipschitz targets.

\subsection{Illustrative example}\label{sect:IllustrativeExample}
To illustrate the intuitive and visually appealing operation of CQnet, consider a classification example of 1D data embedded in 2D (Figure \ref{fig:data_results_1d}). We classify by adding to the network's last layer, $\operatorname{output} = \TW \bx_f$, where $\TW$ is a learnable classifier matrix $\TW \in \mathbb{R}^{n_\text{class} \times n_\text{channels}}$ and $\bx_f$ is the final state of the network. For simplicity, there is no constraint set $C$, all $Q_k$ are half spaces as in the case of the ReLU activation, and the network has $20$ layers with each $\alpha_k = 0.2$.

A trained neural network classifies the data (a point cloud) by transforming the data points into an organization that is linearly separable. CQnet achieves this transformation of the data by moving each point toward the learned constraint sets. Figure \ref{fig:data_results_1d} shows the data, network output, and final classification. 
\begin{figure}[h]
\begin{center}
   \includegraphics[width=0.99\columnwidth]{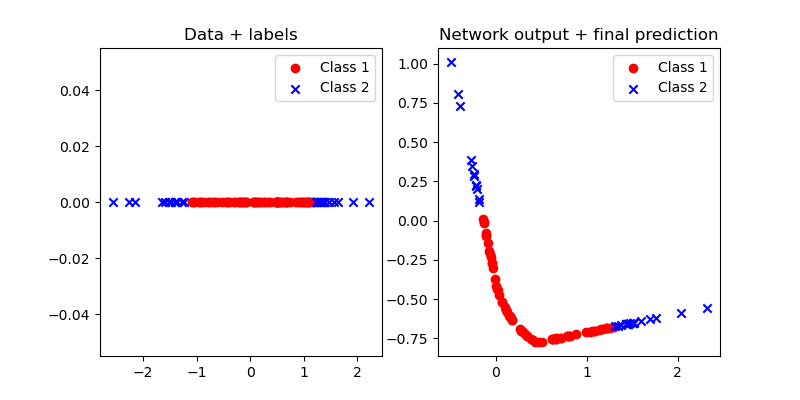}
   \caption{The 1D data are normally distributed numbers and assigned two classes. The network uses one additional dimension for this task. `Bending' the data at the right spot is sufficient to achieve linear separability into the correct classifications.}
   \label{fig:data_results_1d}
   \end{center}
\end{figure}
We are more interested in interpreting how the network arrived at the linearly separable representation. Simply plotting the trajectories of the data (Figure \ref{fig:trajectories_1D_example}) provides a high-level overview, but we obtain more insight when we plot the trajectories along with the evolving constraint set and corresponding distance function that CQnet minimizes at every layer, see Figures \ref{fig:data_trajectories_3D} and \ref{fig:sets_trajectories_1D_example}. For this example, we see that the trajectories generally do not `catch up' to the constraint sets. In a later
section, we show how to enforce such behavior if desired. The figures show a convex-geometric description of the origins of the trajectories.
\begin{figure}[!htb]
 	\centering
 	\begin{subfigure}[b]{0.39\textwidth}
 		\includegraphics[width=\textwidth]{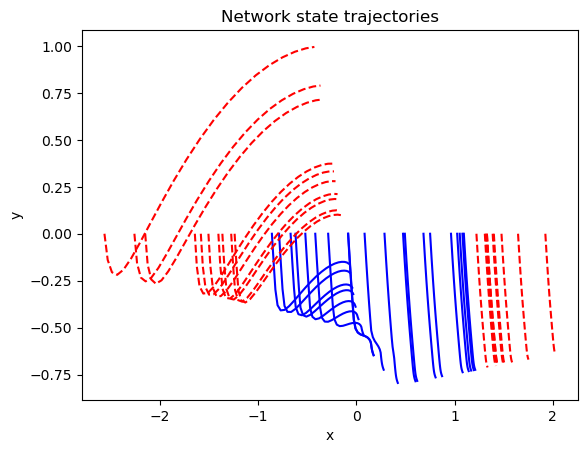}
 		\caption{}
 		\label{fig:Figure2a}
 	\end{subfigure}
 	\begin{subfigure}[b]{0.39\textwidth}
 		\includegraphics[width=\textwidth]{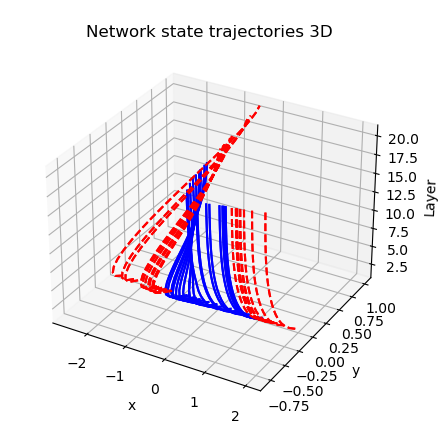}
 		\caption{}
 		\label{fig:Figure2b}
 	\end{subfigure}
 	\caption{The trajectories of the input data as viewed in a feature space and layer-feature space representation.}
\label{fig:trajectories_1D_example}
 \end{figure}

\begin{figure}[htb]
\begin{center}
   \includegraphics[width=0.8\columnwidth]{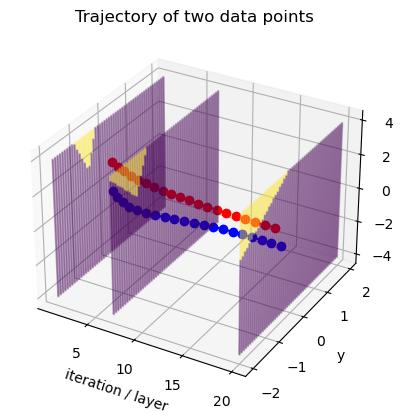}
   \caption{Trajectories for the trained CQnet corresponding to Figures \ref{fig:data_results_1d} and  \ref{fig:trajectories_1D_example}, for one data point from two different classes. Constraint sets for three layers are also displayed.}
   \label{fig:data_trajectories_3D}
   \end{center}
\end{figure}

\begin{figure*}[htb]
\begin{center}
   \includegraphics[width=0.99\textwidth]{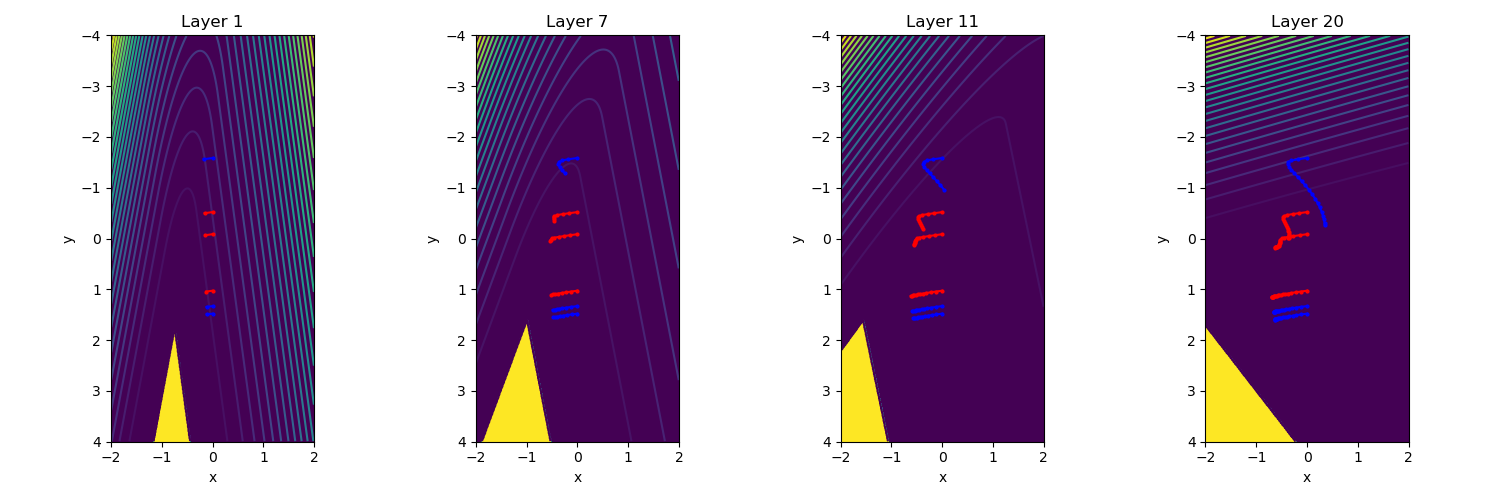}
   \caption{The constraint sets and corresponding squared-distance functions for a few layers of CQnet, and trajectories of a few input data.}
   \label{fig:sets_trajectories_1D_example}
   \end{center}
\end{figure*}

\subsection{Training CQnet}
Training CQnet is no different than many other neural networks. Given labels $\by$ and a loss function $l(\bx,\by)$, training amounts to minimizing the loss between labels and network output
\begin{align}\label{CQnet_training}
\min_{\{\TA_k\}, \TW } \: &l(\TW \bx_f,\by) \:\: \text{s.t.} \\
&P_{C_k} \big( \bx_k - \alpha_k  \TA_k^\top (\Id - P_{Q_k}) \TA_k \bx_k \big) \nonumber \\
&\text{for} \: k=1,2,\cdots,f-1 \quad \text{and} \quad \bx_1 = \bd. \nonumber
\end{align}
The matrix $\TW$ is a linear classifier. Practical minimization can proceed using automatic differentiation. It is implicit that the projection operators have a derivative defined almost everywhere. For instance, the derivative of the projector for bound constraints/ReLU at zero is set to $0$, as in standard deep-learning software.

\subsection{Including a bias term}
Many neural networks standardly incorporate a bias $\bb$ term via an affine transformation $\TA \bx + \bb$. Adding a bias to CQnet is possible without even adjusting the form or notation by using the equivalence $\TA \bx + \bb \leftrightarrow \begin{bmatrix}\TA & \bb \end{bmatrix} \begin{bmatrix}\bx \\ 1\end{bmatrix}$
and substituting it into \eqref{CQnet} to obtain
\begin{align}\label{CQnet_bias1}
&\tilde{\bx}_{k+1}  = P_{\bx_N  = 1}\big( \tilde{\bx}_k - \alpha \tilde{\TA}_k^\top ( \Id  - P_{Q_k} )\tilde{\TA}_k  \tilde{\bx}_k \big)\\ 
&\quad \text{for} \: k=1,2,\cdots,f \quad \text{and} \quad \tilde{\bx}_1 = \begin{bmatrix} \bd \\ 1 \end{bmatrix} \nonumber
\end{align}
where $\tilde{\TA} =  \begin{bmatrix} \TA & \bb \end{bmatrix}$ and $\tilde{\bx} = \begin{bmatrix} \bx \\ 1 \end{bmatrix}$, and the projector $P_{{\bx_N}  = 1}$ ensures the augmented last entry of $\bx$ remains equal to one. 

\subsection{Optional energy conservation/normalization and state shifts}\label{sect:someconstraints}
Two common building blocks for neural networks are normalizations of the state and energy conservation. The latter may arise from physics and control constraints. The former often includes shifting the state back to a mean of zero at each network layer. 

The two above procedures also fit in CQnet via the projection $P_C$.\\
\textbf{Energy conservation} translates to the annulus constraint $\{ \bx \: | \: \sigma_1 \leq \| \bx \|_2 \leq \sigma_2 \}$, where $\sigma_1$ and $\sigma_2$ may be equal to the energy of the input data $\| \bd \|_2$ or somewhat smaller/larger for approximate energy conservation. The projection is known in closed form, but because the set is not convex for $\sigma_1 >0$, the projector is not nonexpansive, and stability is not guaranteed. If used for normalization only, ($\sigma_1 = 0$, $\sigma_2>0$) the set remains convex.

\textbf{Zero-mean states} $\bx_k$ arise by shifting the mean to zero at each network layer. This shift is equivalent to projecting onto the orthogonal complement of the subspace spanned by the vector of all ones $\be = \begin{bmatrix} 1 & 1 & \cdots 1 \end{bmatrix}^\top$. Some basic linear algebra shows that this implies $C = \{ \bx \: | \: \bx \in \operatorname{col}(\be)^\perp \} \leftrightarrow   \{ \bx \: | \: \bx \in \operatorname{null}(\be^\top) \}$, so that the projection onto the set of vectors with zero mean is given by $P_C(\bx) = (\TI - P_{\operatorname{col}(\be)}) \bx  =  (\TI - \be (\be^\top \be)^{-1} \be^\top) \bx \leftrightarrow  P_C(\bx) = \bx - \be \frac{1}{n} \sum_{i=1}^n \bx_i$.

The examples section later on shows other uses for $C$ related to optimal control.

\subsection{Connection to networks as sequences of optimization problems}
The interpretation of CQnet presented in the previous section helps us understand how the states propagate through the network. While the states in each layer track a changing constraint set in general, we can be more precise. For instance, by limiting the change in the matrices $\TA_k$ layer-by-layer. This restriction leads to a more gradually changing constraint set and allows the network states to `catch up' to the constraint sets. A natural way to achieve this is to add regularization to the loss function that promotes slow variation between subsequent matrices $\TA_k$ and $\TA_{k+1}$. Such regularization can also be found in training approaches for networks motivated by differential equations \citep{RuthottoHaber2018}. In our case, we emphasize that the smooth variation of the network weights is not a requirement for stability as shown in Theorem \ref{CQnet_stab}.

The network training (for a single example) with regularization applied to the weights of all $f$ layers amounts to minimizing 
\begin{align}\label{CQnet_training_reg}
\min_{\{\TA_k\} } \:\: &l(\bx_{f},\by) + \frac{\gamma}{2} \sum_{i=1}^{f-1} \| \TA_{i+1} - \TA_{i} \|_F^2 \:\: \text{s.t.}\\
&\bx_{k+1}  =P_{C_k} \big( \bx_k - \alpha_k  \TA_k^\top (\Id - P_{Q_k}) \TA_k \bx_k \big) \nonumber \\
&\text{for} \: k=1,2,\cdots,f-1, \text{ and } \bx_1 = \bd \nonumber,
\end{align}
with penalty parameter $\gamma >0$. Taking the slowly varying weights $\TA_k$ to the extreme leads to constant $\TA_k$ (weight sharing), such that the corresponding part of the network solves (approximately) the SFP $\TA \bx \in Q$ while $\bx \in C$ (assuming all $C_k$ and $Q_k$ are fixed). Then, the full network consists of one or several intervals with constant parameters and is equivalent to solving a sequence of SFPs. Exploring this mode of operation of CQnet is left for future work, but it sets the stage to draw connections to prior work that recognizes some neural networks are equivalent to sequences of sparse-coding problems \citep{JMLR:v18:16-505,8664165}, and methods that seek fixed points defined through non-explicit models like deep equilibrium models \citep{NEURIPS2019_01386bd6} or fixed point networks \citep{heaton2021feasibility}, differentiable optimization layers \citep{pmlr-v70-amos17a,NEURIPS2019_9ce3c52f} and other implicitly defined networks \citep{doi:10.1137/20M1358517}.

\section{Numerical examples}
The implementation and all examples are available at http://... Appendix \ref{AppA} contains an additional example.

\subsection{Multi-Agent control}
CQnet naturally supports states that are always feasible w.r.t. a certain learned/not-learned constraint set while simultaneously approaching other constraint sets. The feasibility for these trajectory constraints holds by construction for training and inference. CQnet can thus solve some (optimal) control problems while simultaneously highlighting its geometrical interpretation. 

Figure \ref{fig:OptControlFig} introduces a discrete time and continuous space multi-agent path-finding problem, a version of which appeared in \cite{9786046}. The task is simultaneously moving two agents from two possible starting areas to two targets.

The main challenges are moving through the corridor by
avoiding the obstacles and staying at a safe distance from
the obstacle while also not colliding with each other by preserving a safe inter-agent distance. The state vector $\bx \in \mathbb{R}^4$ contains the coordinates of the two agents ($\bx_a \in \mathbb{R}^2$ and $\bx_b \in \mathbb{R}^2$). It is clear that any feasible trajectory should satisfy at least the state constraints
\begin{align}\label{control_costraints}
&C \equiv \{ \bx \: | \: \| \bx_a - \bx_b \|_2 \geq 2.0 \}, \: \bx = \begin{bmatrix} \bx_a \\ \bx_b \end{bmatrix} \\
&Q_1 \equiv \{ \bx \: | \: \| \bx - \bx_\text{target} \|_2 \leq \epsilon \} \\
&Q_2 \equiv \{ \bx \: | \: \| \bx - \bx_\text{object,i} \|_2 \geq 1.0 \} \: \forall \:  \bx_\text{object,i}
\end{align}
where $\bx_\text{object,i}$ is one gridpoint of the object to avoid, and $\epsilon$ a small target tolerance. The nonconvex set $C$ describes anti-collision between agents, $Q_1$ is a target constraint, and $Q_2$ is a nonconvex description of object avoidance. 

Now consider an approach that combines ideas from neural ODEs, control, the CQnet, and the CQ algorithm. Below we state an optimal control approach to find trajectories for one or a small set of non-adjustable initial conditions. It uses as the dynamics a mix of CQnet and the CQ algorithm for the above constraints,
\begin{align}\label{CQnet_optcon}
&\min_{\{\TA_t\}} \:\: \beta_1 \sum_{t=1}^{t=T} J(\bx,t) + \beta_2 \| \bx_T - \bx_\text{target} \|_2^2 \:\: \text{s.t.}\\
&\bx_{t+1}  = P_C \big( \bx_t - \sum_{i=1}^2 \alpha_i ( \Id - P_{Q_i}) \bx_t   \nonumber \\
&+ \alpha_3 \TA_t^\top (\Id - P_{Q_3}) \TA_t \bx_t \big) \:\: \text{for} \: t \in \{1,T-1\} \nonumber \\
&\bx_1 = \bx_\text{start} \nonumber
\end{align}
where $P_{Q_3}$ is equivalent to the ReLU as in earlier examples. CQnet provides the dynamics with learnable weights $\TA_t \in \mathbb{R}^{6 \times 4}$ for every layer, as well as contributions for the regular CQ iteration applied to the non-learnable constraints $C$, $Q_1$, and $Q_2$ (see \cite{Censor_2005} for details related to the extension of the SFP to multiple constraint sets). The running costs for each time ($t=1,\cdots,T$) are $J(\bx,t) = \|  \bx_{t+1} - \bx_t \|$. The terminal cost appears as the constraint $Q_1$ and as a loss term. 

CQnet creates states $\bx_t$ that are feasible with respect to the anti-collision constraints at every 'time-step' $t$, while simultaneously moving toward the target and away from the objects (if within the halo). Figure \ref{fig:control_no_network} illustrates that simply running the CQ algorithm (i.e., $\beta_1 = 0$ and $\alpha_3 = 0$) on a nonconvex problem can lead to agents that get stuck.

Learning via \eqref{CQnet_optcon} does not use example trajectories or controls. The results with the neural network and running costs are shown in Figure \ref{fig:control_w_network}. Figure \ref{fig:control_multiple_traj} displays multiple trajectories that show empirically that the agents can make it to the targets from any two starting positions from the initial training conditions. The experiment used a final time of $T=100$. 

For the nonconvex problem, it is also expected that some hyperparameter tuning is required; we set $\beta_1 = 0.15$, $\beta_2 = 1$, $\alpha_1  = 0.1$, $\alpha_2 = 0.5 $, and $\alpha_3  = 0.05$.

\begin{figure*}[!htb]
 	\centering
 	\begin{subfigure}[b]{0.24\textwidth}
 		\includegraphics[width=\textwidth]{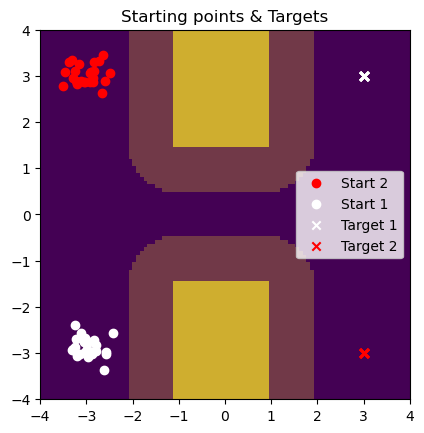}
 		\caption{}
 		\label{fig:exp_corr}
 	\end{subfigure}
 	\begin{subfigure}[b]{0.24\textwidth}
 		\includegraphics[width=\textwidth]{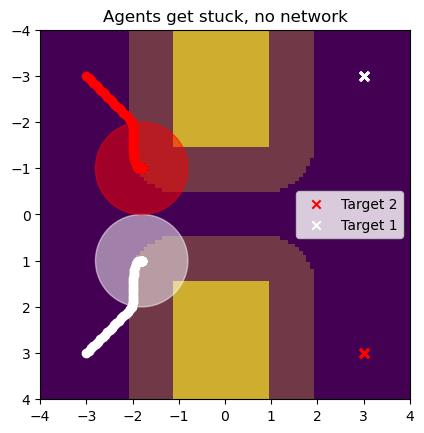}
 		\caption{}
 		\label{fig:control_no_network}
 	\end{subfigure}
\begin{subfigure}[b]{0.24\textwidth}
 	\includegraphics[width=\textwidth]{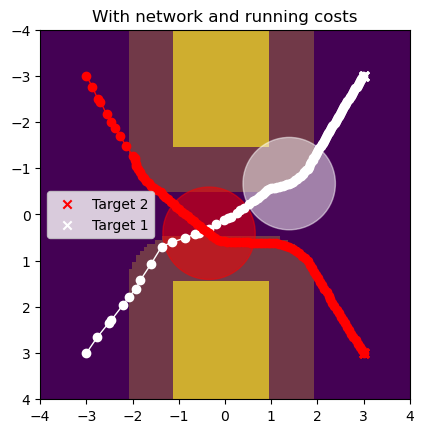}
 		\caption{}
 		\label{fig:control_w_network}
 	\end{subfigure}
 	\begin{subfigure}[b]{0.24\textwidth}
 		\includegraphics[width=\textwidth]{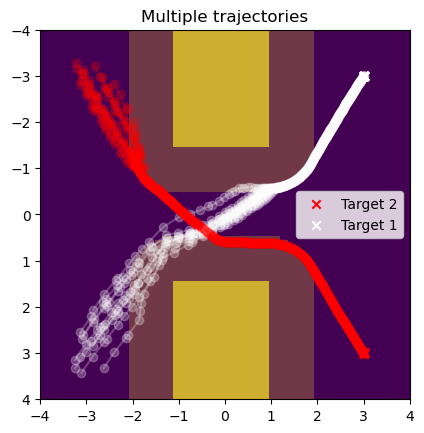}
 		\caption{}
 		\label{fig:control_multiple_traj}
 	\end{subfigure}
 	\caption{(a) Experimental setting with various possible starting locations, two targets, object (solid yellow), object safety halo (shaded). Agents need to stay a distance of $2.0$ apart, which makes it impossible to simultaneously pass through the corridor. (b) Results of running the CQ algorithm (no network or learned parameters) with anti-collision constraints, target ($\ell_2$-ball) constraint, and object avoidance constraint (annulus constraint) as in \eqref{control_costraints}. The agents can get stuck. (c) Results from CQnet with the constraints as in \eqref{control_costraints} and a time-dependent learned constraint that acts as the controller and dynamics. (d) Trajectories are shown for multiple initial conditions (still two agents at a time). All plots show trajectories at final time $T=100$; radius$=1$ circles in (b) and (c) are shown for $T=60$.}
 	\label{fig:OptControlFig}
 \end{figure*}

\subsection{Fashion MNIST}
While not the primary target application of CQnet, this example verifies that CQnet performs similarly to other networks on a simple general-purpose task: fashion-MNIST \citep{xiao2017} classification. We compare against the Resnet \citep{he2016deep}, characterized by the layer $\bx_{k+1} = \bx_k - \alpha \operatorname{ReLU}(\TA_k \bx_k + \bb_k)$; and a variant with the symmetric layer \citep{RuthottoHaber2018} $\bx_{k+1} = \bx_k - \alpha \TA_k^\top \operatorname{ReLU}(\TA_k \bx_k + \bb_k)$ that has more favorable stability properties. We compare these three networks because they have a superficially similar structure. Each network has seven layers with convolutional weight tensors, $\TA_k$, of size $3 \times 3 \times 36 \times 36$ ($36$ channels) in all cases. There is a learnable opening/embedding layer of size $3 \times 3 \times 1 \times 36$, and a final classifier matrix that maps to the number of classes. Average poolings after layers two, four, and six reduce the feature size by a factor of two in each dimension. We look at a bare-bones comparison without data augmentation, no weight regularization, and training uses a fixed learning rate for stochastic gradient descent (batch size of one). The achieved test accuracy was CQnet: $90.91 \%$, ResNet: $90.98 \%$, and ResNet with symmetric layer: $90.89 \%$, averaged over five training runs with different random initializations. So while it is encouraging that CQnet performs similarly, the primary reason for introducing CQnet is to offer a geometrical interpretation and employ it for satisfying trajectory constraints with provable stability properties. 

\section{Conclusions}
This work introduces CQnet: a new neural network design with origins in the CQ algorithm for solving convex split-feasibility problems. The new network design provides novel set-based geometric insights into data points' trajectories when propagating through the network. The network also allows for incorporating multiple constraints on the trajectories that are satisfied at every layer while training and during inference. CQnet also gradually approaches another type of learnable constraint set via descent on its distance function. Because CQnet comprises projection operators and gradients of point-to-set distance functions, it naturally includes layer/state normalization, energy conservation, and certain non-linear activation functions. We provide a stability proof constructed from nonexpansive operators without simplifying assumptions on the variation of network weights per layer, the presence of the activation function, or layer normalization. Examples illustrate how CQnet provides a new convex-geometrical perspective on the internal operation of a neural network and how its constraint-handling capabilities connect it to problems with physical constraints and optimal control.

\normalsize
\bibliography{biblio}

\newpage
\appendix
\onecolumn
\section{Illustrative 2D example.}\label{AppA}
Similar to the illustrative example of 1D data embedded in 2D as in section \ref{sect:IllustrativeExample}, we now show an example of 2D data embedded in 3D (Figure \ref{fig:data_2d}) and the effect of per-sample energy conservation. Training CQ net for the $100$ data points $\bd_i \in \mathbb{R}^3$ (where the third augmented dimension is always $0$) proceeds by minimizing
\begin{align}\label{CQnet_training_E_conservation}
\min_{\{\TA_k\}, \TW } \: &l(\TW \bx_f,\by) \:\: \text{s.t.} \\
&P_{C_k} \big( \bx_k - \alpha_k  \TA_k^\top (\Id - P_{Q_k}) \TA_k \bx_k \big) \nonumber \\
&\text{for} \: k=1,2,\cdots,f-1 \quad \text{and} \quad \bx_1 = \bd. \nonumber
\end{align}
For each data point $\bd_i$ we select all $P_{C_k}$ as the annulus $C_k = \{ \bx \: | \: 0.9\|\bd_i\|_2 \leq \| \bx \|_2 \leq 1.1\|\bd_i\|_2 \}$ to guarantee approximate energy consevation. All $P_{Q_k}$ are the halfspace (ReLU) constraint, all $\TA_k \in \mathbb{R}^{3 \times 3}$, $\TW \in \mathbb{R}^{1 \times 3}$, and the loss function $l(\TW \bx_f,\by)$ is measured as binary cross-entropy.

\begin{figure}[h]
\begin{center}
   \includegraphics[width=0.45\columnwidth]{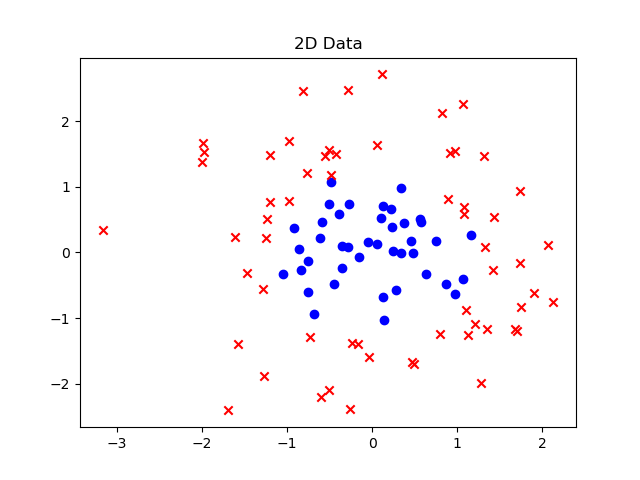}
   \caption{The 2D training data are normally distributed numbers and assigned two classes. The network uses one additional dimension for this task.}
   \label{fig:data_2d}
   \end{center}
\end{figure}

The predictions for all training and validation data is shown in Figures \ref{fig:data2dclassified} and \ref{fig:data2dclassifiedE}, for CQnet with and without energy conservation. Figures \ref{fig:data2dtrajectory} and \ref{fig:data2dtrajectoryE} show the trajectories. The learned constraint sets are not shown because they are layer dependent and are a 4D quantity for this problem.

\begin{figure}[!htb]
 	\centering
 	\begin{subfigure}[b]{0.45\textwidth}
 		\includegraphics[width=\textwidth]{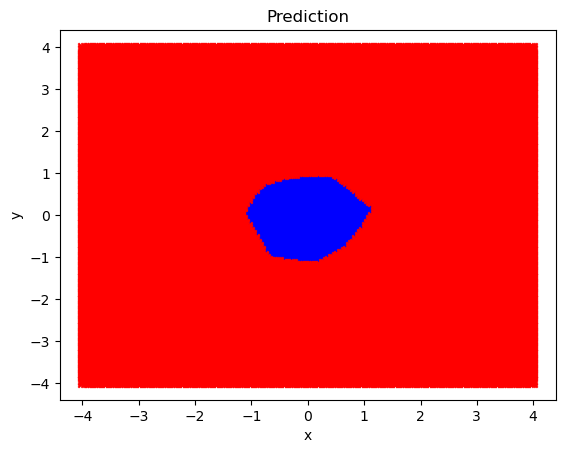}
 		\caption{}
 		\label{fig:data2dclassified}
 	\end{subfigure}
 	\begin{subfigure}[b]{0.45\textwidth}
 		\includegraphics[width=\textwidth]{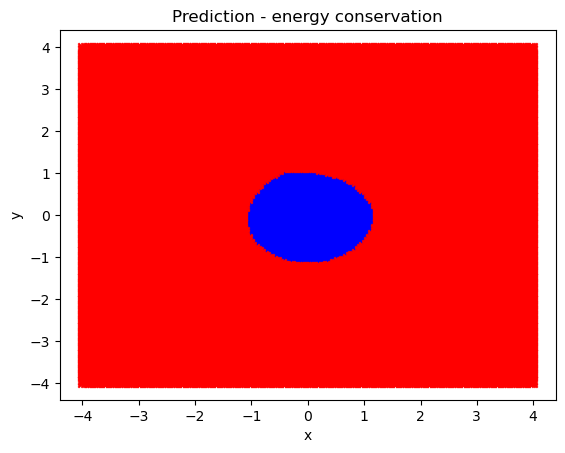}
 		\caption{}
 		\label{fig:data2dclassifiedE}
 	\end{subfigure}
	\begin{subfigure}[b]{0.45\textwidth}
 		\includegraphics[width=\textwidth]{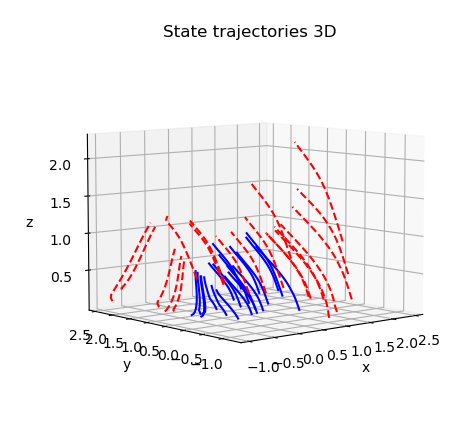}
 		\caption{}
 		\label{fig:data2dtrajectory}
 	\end{subfigure}
 	\begin{subfigure}[b]{0.45\textwidth}
 		\includegraphics[width=\textwidth]{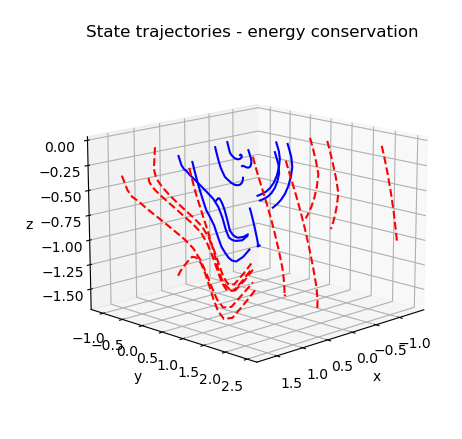}
 		\caption{}
 		\label{fig:data2dtrajectoryE}
 	\end{subfigure}
 	\caption{(a) \& (b) Prediction on all training (Figure \ref{fig:data_2d}) and validation data. The additional energy constraints have a regularizing effect on this problem and increase the prediction accuracy. (c) \& (d) Some trajectories for x-y-z at every layer in the CQnet, after training, for the data in Figure \ref{fig:data_2d}. While somewhat difficult to observe, the trajectories from CQnet with energy conservation are within the annulus.}
\label{fig:all2D_data_classified_trajectories}
 \end{figure}

\end{document}